\documentclass{article}[10pt]
\usepackage[utf8]{inputenc}

\usepackage{microtype}
\usepackage{graphicx}
\usepackage{subfigure}
\usepackage{booktabs} % for professional tables

\usepackage{dsfont}

\usepackage{stackrel}
\usepackage{mathtools}
\usepackage{amsfonts}

\usepackage{bbm}

\usepackage{tikz}
\usetikzlibrary{positioning}
\usetikzlibrary{decorations.text}
\usetikzlibrary{arrows}
\usetikzlibrary{shapes}

\usepackage{pgfplots}
\pgfplotsset{width=10cm,compat=1.9}

\usepackage{amssymb,amsmath}%,amsthm}

\usepackage{hyperref}

\usepackage{xcolor}

\newcommand{\cA}[0]{\mathcal{A}}
\newcommand{\cC}[0]{\mathcal{C}}
\newcommand{\cS}[0]{\mathcal{S}}
\newcommand{\expect}[1]{\mathbb{E}\left[ #1 \right]}

\newcommand{\alice}[0]{\textit{alice}}
\newcommand{\bob}[0]{\textit{bob}}
\newcommand{\eve}[0]{\textit{eve}}
\newcommand{\sm}[0]{\textit{sm}}
\newcommand{\fr}[0]{\textit{fr}}

\allowdisplaybreaks

\title{VC-Dimension Based Generalization Bounds for Relational Learning}

\author{Ond\v{r}ej Ku\v{z}elka\thanks{Department of CS, KU Leuven, Belgium,
Email: ondrej.kuzelka@kuleuven.be} \and Yuyi Wang\thanks{DISCO Group, ETH Zurich, Switzerland, Email: yuwang@ethz.ch} \and Steven Schockaert\thanks{School of CS \& Informatics, Cardiff University, UK, Email: SchockaertS1@cardiff.ac.uk}}

\date{}

\usepackage{amsthm}

\newif\ifappendix
\appendixfalse

\newtheorem{example}{Example}
\newtheorem{definition}{Definition}
\newtheorem{theorem}{Theorem}
\newtheorem{lemma}{Lemma}
\newtheorem{observation}{Remark}

\begin{document}

\maketitle

\begin{abstract}
In many applications of relational learning, the available data can be seen as a sample from a larger relational structure (e.g. we may be given a small fragment from some social network). In this paper we are particularly concerned with scenarios in which we can assume that (i) the domain elements appearing in the given sample have been uniformly sampled without replacement from the (unknown) full domain and (ii) the sample is complete for these domain elements (i.e. it is the full substructure induced by these elements). Within this setting, we study bounds on the error of sufficient statistics of relational models that are estimated on the available data. As our main result, we prove a bound based on a variant of the Vapnik-Chervonenkis dimension which is suitable for relational data.
\end{abstract}

\section{Introduction}

In one of the most common settings in statistical relational learning (SRL), we are given a fragment of a relational structure (i.e.\ a {\em training example}) from which we want to learn a model for making predictions about the unseen parts of the structure. For example, the relational structure could correspond to a large social network and the training example to a fragment of the social network specifying the relationships that hold among a small sample of the users, along with their attributes. Clearly, in order to provide any guarantees on the accuracy of these predictions, we need to make (simplifying) assumptions about how the training structures are obtained. In this paper, we follow the setting from \cite{kuvzelka2017induction,kuzelka2018relational}, where it is assumed that these structures are all obtained as fragments induced by domain elements sampled uniformly without replacement.

%\blue{...}

% An intuitive approach would be to simulate generation of i.i.d.\ subsamples from the global structure by sampling from a suitably selected distribution on subsamples of the given training example \cite{kuzelka2018relational} and apply the classical bounds from statistical learning theory \cite{vapnik_book}. While this strategy is {sometimes} possible, it is sub-optimal from the point of view of statistical power, as is illustrated in the following example.

% \begin{example}\label{ex1}
% Consider a global structure which takes the form of a large directed graph, and assume that we are interested in estimating the probability that the formula $\exists X,Y : \textit{edge}(X,Y)$ holds \blue{for a fragment of the structure induced by} two randomly sampled nodes. Assume furthermore that the given graph was generated by sampling edges independently with some probability $p$. The probability that $\exists X,Y : \textit{edge}(X,Y)$ holds for any two nodes will thus correspond to some value $p^*$ close to $1-(1-p)^2$. As we will see, given a training fragment induced by $n$ nodes from this graph, we can only generate $\lfloor \frac{n}{2} \rfloor$ samples that behave like i.i.d.\ samples. In this case, a more accurate estimate of $p^*$ can be obtained by using all size-$2$ fragments of the training fragment.
% \end{example}

The specific problem that we consider in this paper is to bound the error that we make when estimating probabilities of first-order theories from the training example, or more specifically, the probability that a first-order theory $\Phi$ is satisfied in a small randomly sampled fragment of the relational structure.
While this setting has already been studied in \cite{kuvzelka2017induction,kuzelka2018relational,kuzelka2018arxivPACReasoning},
one important remaining problem, which will be the focus of this paper, relates to how the theory $\Phi$ is obtained. Typically, $\Phi$ is chosen from some hypothesis class, based on the same training example that is used to estimate its probability. The bounds that were derived in \cite{kuzelka2018arxivPACReasoning} for such cases depend on the size of this hypothesis class. Unfortunately, this can quickly lead to vacuous bounds in many cases. In fact, in some applications, the most natural hypothesis classes are either infinite or so large that they are effectively infinite for all practical purposes. This is the case, for instance, whenever we want to use constructs involving numerical expressions. %, such as aggregates (which are e.g.\ commonly used in the context of answer set programming \cite{faber2011semantics}).
To address this issue, in this paper we derive bounds which depend on the VC-dimension of the hypothesis class, instead of its size. In this way, we can also obtain, in many cases, tighter bounds than the ones we derived in \cite{kuzelka2018arxivPACReasoning}. To the best of our knowledge, the bounds we introduce in this paper are the first VC-dimension based bounds for relational learning problems.

%\todo{not to forget to describe the technical challenges}

%*************************************************************************************
\section{Preliminaries}

In this paper we consider function-free language $\mathcal{L}$, which is built from a finite set of constants $\textit{Const}$, a set of variables $\textit{Var}$ and a set of predicates $\textit{Rel} = \bigcup_i \textit{Rel}_i$, where $\textit{Rel}_i$ contains the predicates of arity $i$. Throughout this paper we assume that the sets $\textit{Const}$, $\textit{Var}$ and $\textit{Rel}$ are fixed.
%a fixed first-order language $\mathcal{L}$ is used.
For $a_1,...,a_k \in \textit{Const}\cup \textit{Var}$ and $R \in \textit{Rel}_k$, we call $R(a_1,...,a_k)$ an {\em atom}.  If $a_1,..,a_k\in \textit{Const}$, this atom is called {\em ground}. A {\em literal} is an atom or its negation. A formula is called {\em closed} if all variables are bound by a quantifier.
Note that although the set $\textit{Const}$ is required to be finite, it can have arbitrary size, so that we could, for instance, represent all 64-bit floating point numbers. From an application point of view, this allows us to consider formulas involving numerical expressions. For example, we could have a predicate \textit{Sum}, whose intended meaning is that $\textit{Sum}(x,y,z)$ holds iff $z=x+y$ where $+$ represents floating-point addition.
%Such predicates are particularly useful in combination with so-called aggregates, which are expressions of the form $\textit{op}\{\alpha_1:w_1,...,\alpha_n:w_n\}$, where $\textit{op}$ is a numerical operator (e.g.\ sum, min, max), $\alpha_1,...,\alpha_n$ are formulas and $w_1,...,w_n \in \mathbb{Z}$ are weights. In a given interpretation, $\textit{op}\{\alpha_1:w_1,...,\alpha_n:w_n\}$ evaluates to the integer that is obtained by applying the operator $\textit{op}$ to the weights $w_i$ of the formulas $\alpha_i$ that are satisfied in this interpretation. Note that formulas involving aggregates\footnote{Here, for simplicity, we assume that the aggregates are non-recursive \cite{faber2011semantics}.} can be seen as compact representations for classical formulas, hence we will not explicitly take them into account.

%A possible world $\omega$ is defined as a set of ground atoms. The satisfaction relation $\models$ is defined in the usual way. A substitution is a mapping from variables to terms.

%\subsection{Examples and Probabilities of Formulas}\label{secPreliminariesExamples}
{\subsection{Relational Learning Setting}}\label{secPreliminariesExamples}

\subsubsection{{Relational examples}}

%The concepts described in this section are based on \cite{kuvzelka2017induction,kuzelka2018relational}.
The learning setting considered in this paper follows the one that was introduced in \cite{kuvzelka2017induction,kuzelka2018relational}. The central notion is that of a {\em relational example} (or simply \textit{example} if there is no cause for confusion),%\nb{I was just wondering whether it wouldn't make sense to change the terminology to ``relational example''.},
which is defined as a pair $(\mathcal{A},\mathcal{C})$, with $\mathcal{C}$ a set of constants and $\mathcal{A}$ a set of ground atoms which only use constants from $\mathcal{C}$. A relational example is intended to provide a complete description of a possible world, hence any ground atom over $\mathcal{C}$ which is not contained in $\mathcal{A}$ is implicitly assumed to be false. Note that this is why we have to explicitly specify $\mathcal{C}$, as opposed to simply considering the set of constants appearing in $\mathcal{A}$. For instance, the relational example $(\{\textit{sm}(alice)\},\{\textit{alice}\})$ is different from $(\{\textit{sm}(alice)\},\{\textit{alice},\textit{bob}\})$, as in the latter case we know that \textit{bob} does not smoke (i.e.\ the atom $\textit{sm}(bob)$ is known to be false since it is not specified to be true) whereas in the former case we have no knowledge about \textit{bob}. We denote by $\Omega(\mathcal{L},k)$ the set of all possible relational examples $\Upsilon = (\cA,\cC)$ where $\cA$ {only contains ground atoms} from $\mathcal{L}$ and $|\cC| = k$.

\begin{example}
Let us assume that the only predicate in $\mathcal{L}$ is $\textit{sm}/1$ and the only constant in is $\textit{alice}$. Then $\Omega(\mathcal{L},1) = \{ (\textit{sm}(\textit{alice}),\{\textit{alice}\}), (\emptyset,\{\textit{alice}\}) \}$.
\end{example}

\noindent Let $\Upsilon = (\mathcal{A},\mathcal{C})$ be a relational example and $\mathcal{S}\subseteq \mathcal{C}$.  The fragment $\Upsilon\langle \mathcal{S} \rangle = (\mathcal{B},\mathcal{S})$ is defined as the restriction of $\Upsilon$ to the constants in $\mathcal{S}$, i.e.\ $\mathcal{B}$ is the set of all atoms from $\mathcal{A}$ which only contain constants from $\mathcal{S}$.

\begin{example}\label{ex1.1}
Let
\begin{align*}
    \Upsilon = (\{ \fr(\alice,\bob),\fr(\bob,\alice),\fr(\bob,\allowbreak\eve),\fr(\eve,\bob),\sm(\alice) \}, \\
    \{\alice, \bob, \eve \}),
\end{align*}
i.e.\ the only smoker is $\textit{alice}$ and the friendship structure is:
\begin{center}
\resizebox{0.315\textwidth}{!}{
\tikzset{
main node/.style={ellipse,fill=white!11,draw,minimum size=0.3cm,inner sep=0pt},
other node/.style={rectangle,fill=white!11,minimum size=0.3cm,inner sep=0pt},
}
\tikzset{edge/.style = {->,> = latex'}}
\begin{tikzpicture}

\node[main node] (1) {alice};
\node[main node] (2) [right = 0.5cm of 1] {bob};
\node[main node] (3) [right = 0.5cm of 2] {eve};
%\node[other node] (4) [left = 0.2cm of 1] {sm(alice)};

\draw[edge] (1) [bend right] to (2);
\draw[edge] (2) [bend right] to (1);
\draw[edge] (3) [bend right] to (2);
\draw[edge] (2) [bend right] to (3);
\end{tikzpicture}}
\end{center}
Then $\Upsilon\langle \{ \alice, \bob \} \rangle$ ${=}$ $(\{\sm(\alice),$ $\fr(\alice,\bob),$ $\fr(\bob,\alice) \},$ $\{\alice,\bob\})$.
\end{example}

\noindent In the considered setting, we are given a single relational example $\Upsilon = (\mathcal{A},\mathcal{C})$, and this example is assumed to have been sampled from a larger relational example $\aleph = (\cA_\aleph,\cC_\aleph)$. The intended meaning is that $\aleph$ covers the entire domain which we would like to model and $\Upsilon$ is the fragment of the domain which is known at training time. Throughout this paper, we will assume that $\cC_\aleph$ is finite.
As in \cite{kuzelka2018relational,kuzelka2018arxivPACReasoning} we assume that $\Upsilon$ as sampled from $\aleph$ by the following process.

\begin{definition}[Sampling from a global example]\label{def:sampling_setting}
Let $\aleph = (\cA_\aleph,\cC_\aleph)$ be a relational example called the {\em global example}. Let $n \in \mathbb{N}\setminus\{0\}$ and let $\textit{Unif}(\cC_\aleph,n)$ denote uniform distribution on size-$n$ subsets of $\cC_\aleph$. Training relational examples $\Upsilon$ are sampled from the global example $\aleph$ by first sampling $\cC_\Upsilon \sim \textit{Unif}(\cC_\aleph,n)$ and defining $\Upsilon = \aleph \langle \cC_\Upsilon \rangle$.
\end{definition}

%\noindent Here, the global example represents the state of the complete system, for instance a social network, that is not fully observed.

\subsubsection{Probabilities of formulas}
In a given relational example, any closed formula $\alpha$ is classically either true or false. To assign probabilities to formulas in a meaningful way, considering that we typically only have a single relational example available for training, we can consider how often the formula is satisfied in small fragments of the given relational example.

\begin{definition}[Probability of a formula \cite{kuzelka2018relational}]\label{def:probability_of_formula}
Let $\Upsilon = (\mathcal{A},\mathcal{C})$ be a relational example and $k\in \mathbb{N}$.
The probability of a closed formula $\alpha$ is defined as follows\footnote{We will use $Q$ for probabilities of formulas as defined in this section, to avoid confusion with other ``probabilities'' we deal with in the text.}:
$$Q_{\Upsilon,k}(\alpha) = P_{\mathcal{S} \sim \textit{Unif}(\cC,k)}\left[ \Upsilon\langle \mathcal{S} \rangle \models \alpha \right]$$
where $\textit{Unif}(\cC,k)$ denotes uniform distribution on size-$k$ subsets of $\cC$.
\end{definition}

\noindent Clearly $Q_{\Upsilon,k}(\alpha) = \frac{1}{|\mathcal{C}_k|} \cdot \sum_{\mathcal{S} \in \cC_k} \mathds{1}(\Upsilon\langle \mathcal{S} \rangle \models \alpha)$ where $\mathcal{C}_k$ is the set of all size-$k$ subsets of $\cC$.
The above definition can straightforwardly be extended to probabilities of sets of formulas (which we will also call {\em theories} interchangeably): if $\Phi$ is a set of formulas, we then have
%\nb{Why use $\myeq$ here (but not elsewhere in the paper)?}
$Q_{\Upsilon,k}(\Phi) = Q_{\Upsilon,k}(\bigwedge \Phi)$ where $\bigwedge \Phi$ denotes the conjunction of all formulas in $\Phi$.

\begin{example}
Let $\textit{sm}/1$ be a unary predicate denoting that someone is a smoker, e.g. $\textit{sm}(\textit{alice})$ means that $\alice$ is a smoker. Let us consider the following example:
\begin{align*}
\Upsilon = (\{ \textit{fr}(\textit{alice}, \textit{bob}), \textit{sm}(\textit{alice}), \textit{sm}(\textit{eve}) \},
    \{\textit{alice}, \textit{bob}, \textit{eve} \}),
\end{align*}
and formulas $\alpha = \forall X : \textit{sm}(X)$ and $\beta = \exists X,Y : \textit{fr}(X,Y)$. Then, for instance, $Q_{\Upsilon,1}(\alpha) = 2/3$, $Q_{\Upsilon,2}(\alpha) = 1/3$ and $Q_{\Upsilon,2}(\beta) = 1/3$.
\end{example}

\noindent It is not difficult to check that under the sampling assumption from Definition \ref{def:sampling_setting}, for any theory $\Phi$ it holds that $Q_{\aleph,k}(\Phi) = \mathbb{E}_{\Upsilon}\left[ Q_{\Upsilon,k}(\Phi) \right]$ \cite{kuzelka2018relational}. %\nb{Removed the sentence about functions here.}
%The same also holds for any function $f : \Omega(\mathcal{L},k) \rightarrow \{0,1\}$. %Test examples are assumed to be sampled using the same process as training examples, except the size of the sampled set of constants can be different.

%\subsubsection{Functions as Succinct Representations of Theories}
\subsubsection{Representing theories as functions}

By definition, to compute $Q_{\Upsilon,k}(\Phi)$, we only need to know for which of the elements of $\Omega(\mathcal{L},k)$ it holds that $\Phi$ is satisfied. To make this view explicit, we will formulate the results in this paper in terms of functions from $\Omega(\mathcal{L},k)$ to $\{0,1 \}$. For a given theory $\Phi$, the associated function $f_{\Phi}$ is defined for $\Gamma\in \Omega(\mathcal{L},k)$ as $f_{\Phi}(\Gamma)=1$ if $\Gamma\models\Phi$ and $f_{\Phi}(\Gamma)=0$ otherwise. The advantage of this formulation is that our results then directly apply to settings where other representation frameworks than classical logic are used for representing the theory. For example, a theory could be implicitly represented by a neural network
%The value of $Q_{\Upsilon,k}(\Phi)$ depends only on the behavior of $\Phi$ on size-$k$ relational structures described using the first-order language $\mathcal{L}$, i.e. on the behaviour on $\Omega(\mathcal{L},k)$. We can also consider a possibly more compact representation of theories using arbitrary functions from $\Omega(\mathcal{L},k)$ to $\{0,1 \}$. For instance, such theories can be represented by neural networks
with a hard-thresholding output unit. For notational convenience, we also write $\Gamma \models f$ if $f(\Gamma) = 1$. We then naturally extend the definition of $Q_{\Upsilon,k}$ to functions: $Q_{\Upsilon,k}(f) = P_{\mathcal{S} \sim \textit{Unif}(\cC,k)}\left[\Upsilon\langle \mathcal{S} \rangle \models f  \right] = P_{\mathcal{S} \sim \textit{Unif}(\cC,k)}\left[ f(\Upsilon\langle \mathcal{S} \rangle) = 1  \right].$

\subsection{VC-Dimension}\label{sec:vc}

The next definition describes the classical notion of VC-dimension \cite{vapnik_book}, specialized to our relational learning setting
%\nb{This formulation is a bit confusing. At first I thought that the notion of VC-dimension that was going to be used in this paper was not the usual one.}
that is used throughout this paper to measure {the} complexity of hypothesis classes.

\begin{definition}[VC-dimension]
Let $k$ be a positive integer and let $\mathcal{H}$ be a hypothesis class of functions $f : \Omega(\mathcal{L},k) \rightarrow \{0,1\}$. Let $\mathcal{X} = \{ \Upsilon_1, \Upsilon_2,\dots,\Upsilon_d \}\subseteq \Omega(\mathcal{L},k)$. We say that $\mathcal{H}$ shatters $\mathcal{X}$ if for every $\mathcal{Y} \subseteq \mathcal{X}$, there is $f \in \mathcal{H}$ such that $f(\Upsilon) = 1$ for all $\Upsilon \in \mathcal{Y}$ and $f(\Upsilon) = 0$ for all $\Upsilon \in \mathcal{X}\setminus \mathcal{Y}$. The VC dimension of $\mathcal{H}$ is the largest integer $d$ such that there exists a subset of $\Omega(\mathcal{L},k)$ with cardinality $d$ that is shattered by $\mathcal{H}$.
\end{definition}

\noindent The next definition formalizes what we mean when we say that two functions are equivalent w.r.t. a given global example.

\begin{definition}
We say two functions $f$ and $g$ are $k$-equivalent w.r.t.\ a global example $\aleph$ if for any size-$k$ set $\cS$ it holds that $f(\aleph\langle\cS\rangle) = g(\aleph\langle\cS\rangle)$.
\end{definition}

\noindent Naturally the above two definitions can also be applied to theories, e.g.\ two theories $\Phi$ and $\Theta$ are $k$-equivalent w.r.t.\ a global example $\aleph$ if their associated functions $f_{\Phi}$ and $f_{\Theta}$ are $k$-equivalent.
The following observation will play an important role in the proofs.
\begin{observation}\label{remark:finitely_many_hypotheses}
The maximum number of hypotheses that are mutually non-equivalent w.r.t.\ a given (finite) global example $\aleph$ is finite.
\end{observation}
A consequence of this observation is that even for infinite hypothesis classes, in principle, there are only finitely many different hypotheses that need to be considered. However, given that we typically do not know the size of the global example, in practice it is not possible to rely on the number of non-equivalent hypotheses to apply the bounds from \cite{kuzelka2018arxivPACReasoning} to infinite hypothesis classes. In contrast, the bounds that we introduce in this paper can still be applied in such cases, as long as the hypothesis class has a finite VC-dimension.

The ability to deal with infinite hypothesis classes makes it possible, for instance, to learn theories based on differentiable architectures \cite{lrnns,DBLP:conf/nips/Rocktaschel017} or based on graph kernels \cite{vishwanathan2010graph}.

%\section{Applications: Relational Marginal Problems, PAC-Reasoning}
%******************************************************************************
\section{Motivation}
The main aim of this paper is to derive bounds on how accurately we can estimate $Q_{\aleph,k}(f)$ from a given training relational example $\Upsilon$, where $f$ is viewed as a logical formula. The need for such probability estimates naturally arises, among others, in the setting of relational marginal problems, which were studied in \cite{kuzelka2018relational}. In that setting, we are given a set of formulas $\Theta = \{\alpha_1,\dots,\alpha_{|\Theta|} \}$, a set of constants $\mathcal{C}$ and a training relational example $\Upsilon = (\cA_\Upsilon,\cC_\Upsilon)$. The task is to use the probabilities of $\alpha_1,\dots,\alpha_{|\Theta|}$ that are estimated from the training relational example $\Upsilon$ to perform inference on the domain $\mathcal{C}$. Specifically, the task is to find a maximum entropy distribution on the set of all relational examples of the form $\Psi = (\cA_\Psi, \cC)$, such that $\mathbb{E}[Q_{\Psi,k}(\alpha_i)] = \widehat{Q}_{\Upsilon,k}(\alpha_i)$ for all $\alpha_i \in \Theta$.
%Here, $\widehat{Q}_{\Upsilon,k}(\alpha_i)$ is an ({\em adjusted}) estimate of $\mathbb{E}[Q_{\Psi,k}(\alpha_i)]$ based on $Q_{\Upsilon,k}(\alpha_i)$.
Here, $\widehat{Q}_{\Upsilon,k}(\alpha_i)$ is an estimate of $\mathbb{E}[Q_{\Psi,k}(\alpha_i)]$ which is based on $Q_{\Upsilon,k}(\alpha_i)$. If $|\cC| \leq |\cC_\Upsilon|$ then this estimate is simply given by $\widehat{Q}_{\Upsilon,k}(\alpha_i) = Q_{\Upsilon,k}(\alpha_i)$.
%\nb{Right?}.
In general, however, the value $Q_{\Upsilon,k}(\alpha_i)$ needs to be adjusted to account for the difference in the size of the training relational example domain $\cC_\Upsilon$ and the domain $\mathcal{C}$ over which we want to perform inference. The resulting distribution is similar to a Markov logic network, and can be used in applications for similar purposes\footnote{The relational marginal problems that we consider in this paper are referred to as Model A in \cite{kuzelka2018relational}. Another type of relational marginal problems, referred to as Model B in \cite{kuzelka2018relational}, leads to distributions that are exactly Markov logic networks.}; it is an exponential family distribution of the following form:
 $$
 P(\Psi) = \frac{1}{Z} \exp{\left( \sum_{\alpha_i \in \Theta} w_i \cdot Q_{\Psi,k}(\alpha_i) \right)}.
 $$
%where $\#_k(\alpha_i,\Psi)$ is the fraction of size-$k$ sets $\cS \subseteq \steven{\cC}$ such that $\Psi\langle \cS \rangle \models \alpha_i$ \nb{So can't we just write $Q_{\Psi,k}(\alpha_i)$ instead of $\#_k(\alpha_i,\Psi)$?}.
In the case $|\cC| = |\cC_\Upsilon|$, the weights $w_i$ can be obtained by solving a maximum likelihood problem which is the dual of the maximum entropy problem.
%can be obtained by solving the maximum entropy problem, whose dual is similar to maximum likelihood (the two are equivalent when $|\cC| = |\cC_\Upsilon|$).
Ideally, we would use $Q_{\aleph,k}(\alpha_i)$ as the estimates of $Q_{\Psi,k}(\alpha_i)$ in the maximum entropy problems. Since, in reality, we do not have {access to} $Q_{\aleph,k}(\alpha_i)$, we {need} to use the estimates based on $Q_{\Upsilon,k}(\alpha_i)$. {The results we present in this paper shed light on the impact of this simplification.} We refer the reader to \cite{kuzelka2018relational} for more details.
%\todo{What is still missing in this paragraph are a few sentences explaining how all of this actually relates to this paper. It talks about using the values $Q_{\Upsilon,k}(\alpha_i)$, but there's no explanation as to why we'd want to know how accurately $Q_{\Upsilon,k}(\alpha_i)$ approximates $Q_{\aleph,k}(\alpha_i)$ (or why we'd care about $Q_{\aleph,k}(\alpha_i)$ at all).}

%In this section we briefly describe two learning settings for which the results presented in this paper directly provide bounds: relational marginal problems from \cite{kuzelka2018relational} and a setting based on PAC-reasoning from \cite{kuzelka2018arxivPACReasoning}. Our results can also be modified for Markov logic networks learning, which we discuss in more detail in Section \ref{sec:generalizations}. Both of the settings rely on the same sampling process from Definition \ref{def:sampling_setting}.

Estimates of $Q_{\aleph,k}(f)$ also play a central role in the analysis of PAC-reasoning \cite{juba,valiant_knowledge_infusion} for relational domains as studied in \cite{kuzelka2018arxivPACReasoning}. This analysis also relies on the  sampling assumptions from Definition \ref{def:sampling_setting}. Specifically, in that setting, a training relational example $\Upsilon$ and a test relational example $\Psi$ are sampled from $\aleph$ and the learner's task is to find a set of first-order logic formulas that will not produce too many errors on $\Psi$ {when} using a restricted form of classical reasoning.
%\todo{Add a sentence about why we'd want to do that} -- If we have space %As for obtaining guarantees for this problem,
To obtain guarantees on the number of literals that are incorrectly inferred using this form or reasoning, we essentially need to bound the difference of $Q_{\Upsilon,k}(\Phi)$ and $Q_{\aleph,k}(\Phi)$ (which allows us to bound the difference with $Q_{\Psi,k}(\Phi)$), which is exactly the problem we also study in this paper. In contrast to \cite{kuzelka2018arxivPACReasoning}, however, we are interested in bounds that are based on the VC-dimension of the hypothesis space.

%over all theories from a given hypothesis set.

 %These bounds can be used for bounding the error encountered in relational marginal problems where the learner selects the formulas $\alpha_i$ from some given class, as well as for bounding the error made in the setting of PAC-reasoning in relational domains. Unlike the bounds derived in \cite{kuvzelka2017induction,kuzelka2018arxivPACReasoning}, which only depend on the size of $\mathcal{H}$, the bounds derived in this paper depend on VC dimension of $\mathcal{H}$ which has an important consequence that the new bounds remain valid even for infinite hypothesis classes,

%****************************************
\section{Summary of the Results}
%In this paper we derive bounds on the random variable $\sup_{f \in \mathcal{H}}\left| Q_{\aleph,k}(f) - Q_{\Upsilon,k}(f) \right|$ where $f$ is a function from $\Omega(\mathcal{L},k)$ to $\{0,1 \}$ and $\mathcal{H}$ is a hypothesis class with {\em finite VC dimension} $d$.

Intuitively, what we need to find is a suitable bound on the quantity $|Q_{\aleph,k}(f) - Q_{\Upsilon,k}(f)|$, i.e.\ we want to bound the error we make when estimating the overall probability of $f$ (i.e.\ the value $Q_{\aleph,k}(f)$) from a training fragment of the global example. In most application settings, however, $f$ itself is also chosen using the training relational example $\Upsilon$, e.g.\ by choosing the hypothesis $f$ that maximizes $Q_{\Upsilon,k}(f)$ among the functions from some hypothesis class $\mathcal{H}$. This means that we cannot find a suitable bound for $|Q_{\aleph,k}(f) - Q_{\Upsilon,k}(f)|$ without taking the hypothesis class $\mathcal{H}$ into account. The classical solution, which we will also follow, is to instead bound the quantity $\sup_{f \in \mathcal{H}}\left| Q_{\aleph,k}(f) - Q_{\Upsilon,k}(f) \right|$. The main result of this paper takes the form of two theorems that provide probabilistic bounds on this latter quantity. The proof of these theorems is presented in Section \ref{sec:proofs}.
%In this section we describe the main results of this paper, %Theorem \ref{lemma:expected_error} and Theorem \ref{thm:thm_main},
%\steven{which take the form of two theorems} that give us probabilistic bounds for the quantity $\sup_{f \in \mathcal{H}}\left| Q_{\aleph,k}(f) - Q_{\Upsilon,k}(f) \right|$.
%These theorems are proved in Section \ref{sec:proofs}.

The first theorem bounds the expected value of
%the random variable $\sup_{f \in \mathcal{H}}\left| Q_{\aleph,k}(f) - Q_{\Upsilon,k}(f) \right|$.
$\sup_{f \in \mathcal{H}}\left| Q_{\aleph,k}(f) - Q_{\Upsilon,k}(f) \right|$ when $\cC_\Upsilon$ is viewed as a random variable.
Interestingly, this bound is essentially the same as the classical bound for the i.i.d. setting \cite{shalev2014understanding}, except that the value of $n$ from the classical bound is replaced by $\lfloor n/k \rfloor$, which is perhaps not surprising as it is the maximum number of non-overlapping size-$k$ subsets of $C_\Upsilon$.

%\todo{We need to decide which form of expected error bound to use and then to also plug it into Thm 2 for E[X].}

\begin{theorem}\label{lemma:expected_error}
Let $\aleph = (\cA_\aleph, \cC_\aleph)$ be a global example and $\cC_\Upsilon$ be sampled uniformly from all size-$n$ subsets of $\cC_\aleph$ and let us define $\Upsilon = \aleph \langle \cC_\Upsilon \rangle$. Then for any hypothesis class $\mathcal{H}$ of functions $f : \Omega(\mathcal{L},k) \rightarrow \{0,1\}$ with finite VC-dimension $d$, the following holds:
\begin{align*}
% \mathbb{E}\left[\sup_{f\in \mathcal{H}}\left|Q_{\aleph,k}(f)-Q_{\Upsilon,k}(f)\right|\right] \leq  2 \sqrt{\frac{\log \mathcal{S}(\mathcal{H},\lfloor n/k \rfloor) + \log 2}{\lfloor n/k \rfloor}}
\mathbb{E}\left[\sup_{f\in \mathcal{H}}\left|Q_{\aleph,k}(f)-Q_{\Upsilon,k}(f)\right|\right] \leq 2 \cdot \sqrt{\frac{2d \log \left(2 e \lfloor n/k \rfloor / d \right)}{\lfloor n/k \rfloor}}
\end{align*}
\end{theorem}

\noindent The second theorem provides a tail bound for $P\left[\sup_{f \in \cal{H}} \left| Q_{\aleph,k}(f) - Q_{\Upsilon,k}(f)  \right| \ge \varepsilon \right]$. We note that the bound on expected error from Theorem \ref{lemma:expected_error} cannot be derived from Theorem \ref{thm:thm_main}, although a different bound on expected error with looser constants could be derived from Theorem \ref{thm:thm_main}.

\begin{theorem}\label{thm:thm_main}
Let $\aleph = (\cA_\aleph, \cC_\aleph)$ be a global example and $\cC_\Upsilon$ be sampled uniformly from all size-$n$ subsets of $\cC_\aleph$ and let us define $\Upsilon = \aleph \langle \cC_\Upsilon \rangle$. Then for any hypothesis class $\mathcal{H}$ of functions $f : \Omega(\mathcal{L},k) \rightarrow \{0,1\}$ with finite VC-dimension $d$, the following holds for any $0 < \varepsilon \leq 1$:
% \begin{multline*}
%  %\le  \cdot \exp\left(-\frac{\lfloor n/k \rfloor\varepsilon^2}{8}\right)
% P\left[\sup_{f \in \cal{H}} \left| Q_{\aleph,k}(f) - Q_{\Upsilon,k}(f)  \right| \ge \varepsilon \right] \\
% \leq 4 \left(\frac{2e\lfloor \frac{n}{k} \rfloor}{d}\right)^d \cdot (1+(2e)^{1/4}) \cdot \exp{\left( - \left\lfloor \frac{n}{k} \right\rfloor \cdot \frac{\varepsilon^2}{32 e^2} \right)} \\
% +\frac{\varepsilon \cdot \left\lfloor \frac{n}{k} \right\rfloor \cdot \max \left\{1, 2 \cdot \sqrt{\frac{2d \log \left(2 e \lfloor n/k \rfloor / d \right)}{\lfloor n/k \rfloor}}  \right\}}{16 e^2} \cdot \exp{\left( - \left\lfloor \frac{n}{k} \right\rfloor \cdot \frac{\varepsilon^2}{16 e^2} \right)}.
% \end{multline*}
\begin{multline*}
    P\left[\sup_{f \in \cal{H}} \left| Q_{\aleph,k}(f) - Q_{\Upsilon,k}(f)  \right| \ge \varepsilon \right] \\
    \leq \exp{\left(- \frac{\lfloor n/k \rfloor \varepsilon^2}{4}\right)} + \varepsilon \sqrt{8 \pi \lfloor n/k \rfloor} \left(\frac{2e \lfloor n/k \rfloor}{d} \right)^d \cdot \exp{\left( - \frac{\lfloor n/k \rfloor \varepsilon^2}{8} \right)}
\end{multline*}
\end{theorem}

\noindent Up to somewhat looser constants, the tail bound from Theorem \ref{thm:thm_main} can be shown to also have the same form as the existing VC tail bounds \cite{vapnik1971uniform}. In particular, the bound implies the following simpler, albeit looser bound:
\begin{multline*}
 %\le  \cdot \exp\left(-\frac{\lfloor n/k \rfloor\varepsilon^2}{8}\right)
P\left[\sup_{f \in \cal{H}} \left| Q_{\aleph,k}(f) - Q_{\Upsilon,k}(f)  \right| \ge \varepsilon \right] \\
\leq \left(1 + \sqrt{8 \pi \lfloor n/k \rfloor} \left(\frac{2e \lfloor n/k \rfloor}{d} \right)^d \right)\cdot \exp{\left( - \frac{\lfloor n/k \rfloor \varepsilon^2}{8} \right)}.
%12 \cdot \left(\frac{3e\lfloor \frac{n}{k} \rfloor}{d}\right)^d \cdot \exp{\left( - \left\lfloor \frac{n}{k} \right\rfloor \cdot \frac{\varepsilon^2}{32 e^2} \right)}.
\end{multline*}

\section{Related Work}

%\todo{pac-learning from noniid data}

There have been several works studying theoretical properties of various statistical relational learning settings.
Dhurandhar and Dobra \cite{dhurandhar2012distribution} derived Hoeffding-type inequalities for classifiers trained with relational data. However, there are several important differences with our work. First, their bounds are not VC-type bounds. Moreover, their results, based on restricting the independent interactions of data points, cannot be applied in our setting, which is more general than the one they consider. %Their setting is also more restricted than ours.
Certain other statistical properties of learning have also been studied for SRL models. For instance, Xiang and Neville \cite{xiang2011relational} studied consistency of estimation in a certain relational learning setting.

From a different perspective, abstracting from the relational logic setting, our results can also be seen as bounds for uniform deviations of U-statistics \cite{hoeffding1948class} under sampling {\em without} replacement. Not many results are known for this particular setting in the literature. One exception is the work of Nandi and Sen \cite{nandi1963properties} who only derived bounds on variance in this setting. It is not possible to derive our results from theirs. In particular, we need Chernoff-type bounds whereas the variance bounds from their work would only give us Chebyshev-type bounds. A more thoroughly studied setting is the estimation of U-statistics under sampling {\em with} replacement. Cl\'{e}mencon, Lugosi and Vayatis \cite{clemenccon2008ranking} derived among others\footnote{The main results of \cite{clemenccon2008ranking} are bounds that assume a certain 'low-noise' condition. Although they only derived bounds for the case $k = 2$ (in our notation), the results directly related to ours can be extended for larger $k$'s as well.} VC-inequalities in a setting similar to ours, but under sampling with replacement, which makes their analysis simpler. %The assumption on sampling with replacement makes their analysis simpler.
However, such an assumption would not make sense in the relational learning setting where it would mean, for instance, that we would end up with multiple copies of the same individual (e.g.\ ending up with social networks in which the same person can occur multiple times).

%\cite{clemenccon2008ranking}: not generalization error bound, only kernels of order two, using a special noise condition
%\todo{Hoeffding-type bounds: Lovasz, Variance bounds: Nandi and Sen}

%Here we also note that it is not always possible or desirable in practice to sample sets of domain elements uniformly as we assumed to be the case in our analysis. Other sampling designs for relational data were studied e.g. in \cite{ahmed2012network}. A study of PAC guarantees for such other sampling designs is left as a topic for future work.

\section{{Derivation of the Bounds}}\label{sec:proofs}

In this section, we prove Theorems \ref{lemma:expected_error} and \ref{thm:thm_main} using a series of lemmas. First, in Section \ref{sec:extracting}, we define a sampling process for generating vectors containing $\lfloor n/k \rfloor$ size-$k$ fragments of $\Upsilon$. The sampling process has two important properties. First, the fragments in each of the vectors are distributed as size-$k$ fragments sampled i.i.d. from $\aleph$ (assuming $\Upsilon$ is sampled as in Definition \ref{def:sampling_setting}). Second, the average of the estimates of $Q_{\aleph,k}(f)$ computed from the vectors converges to $Q_{\Upsilon,k}(f)$. These two properties allow us to use the sampling process to derive a bound on expected value of the random variable $\sup_{f \in \mathcal{H}}\left| Q_{\aleph,k}(f) - Q_{\Upsilon,k}(f) \right|$ in Section \ref{sec:expected}, which finishes the proof of Theorem \ref{lemma:expected_error}.

The proof of Theorem \ref{thm:thm_main} is a bit more involved. First, in Section \ref{sec:characteristic_from_tail}, we derive bounds on the moment-generating function of a random variable that can be obtained if we only know its tail bounds. Then, in Section \ref{sec:tail_from_characteristic}, we combine the results from the preceding sections to prove Theorem \ref{thm:thm_main}. In particular, we use the bound moment-generating function to obtain a tail bound on the estimates of $Q_{\aleph,k}(f)$ by exploiting a trick that is sometimes called {\em average of sums-of-i.i.d blocks} \cite{clemenccon2008ranking}. %\todo{I need to finish this paragraph.}

\subsection{Extracting Independent Samples}\label{sec:extracting}

% We start with a lemma from \cite{kuzelka2018relational}, which we reprove in more detail in the appendix\footnote{The original paper \cite{kuzelka2018relational} contains only a sketch of the full proof. Here we provide all details for completeness.}.

In this section we describe a sampling process that allows us to obtain $\lfloor n/k \rfloor$ samples from $\Upsilon$ that are distributed as i.i.d. samples from $\aleph$, assuming $\Upsilon$ is sampled as in Definition \ref{def:sampling_setting}.

\begin{lemma}\label{lemma:many_vectors}
Let $\aleph = (\cA_\aleph, \cC_\aleph)$ be a global example. Let $0 \leq n \leq |\cC_\aleph|$, $q \geq 1$ and $1 \leq k \leq n$ be integers. Let $\mathbf{X} = (\cS_1,\cS_2,\dots,\cS_{\lfloor \frac{n}{k} \rfloor})$ be a vector of subsets of $\cC_\aleph$, each sampled uniformly and independently of the others from all size-$k$ subsets of $\cC_\aleph$. Next let $\mathcal{I}' = \{ 1,2,\dots,|\cC_\aleph| \}$ and let $\mathbf{Y}_j = (\cS_{j,1}',\cS_{j,2}',\dots,\cS_{j,\lfloor \frac{n}{k} \rfloor}')$, for $1 \leq j \leq q$, be vectors sampled by the following process:
\begin{enumerate}
    \item Sample $\mathcal{C}_\Upsilon$ uniformly from all size-$n$ subsets of $\mathcal{C}_\aleph$.
    \item For $j$ from $1$ to $q$:
    \begin{enumerate}
        \item Sample subsets $\mathcal{I}_1',\dots,\mathcal{I}_{\lfloor \frac{n}{k} \rfloor}'$ of size $k$ from $\mathcal{I}'$.
        \item Sample an injective function $g : \bigcup_{i=1}^{\lfloor n/k \rfloor} \mathcal{I}_i' \rightarrow \cC_\Upsilon$ uniformly from all such functions.
        \item Define $\cS_{j,i}' = g(\mathcal{I}_i')$ for all $0 \leq i \leq \lfloor \frac{n}{k} \rfloor$.
    \end{enumerate}
\end{enumerate}
Then the following holds:
\begin{enumerate}
    \item The random vectors $\mathbf{X}$ and $\mathbf{Y}_j$ have the same distribution for any $1 \leq j \leq q$.
    \item For any function $f : \Omega(\mathcal{L},k) \rightarrow [0,1]$ it holds:
    \begin{align*}
        P\left[ \left|Q_{\Upsilon,k}(f) - \frac{1}{q \lfloor n/k \rfloor} \sum_{j=1}^q\sum_{i=1}^{\lfloor n/k \rfloor} f\left(\Upsilon\langle\cS_{j,i}'\rangle \right) \right| \geq \epsilon \right] \leq 2 \exp{\left( - 2q \varepsilon^2 \right)}
    \end{align*}
\end{enumerate}
\end{lemma}
\begin{proof}
The first part of the proof follows immediatelly from Lemma 3 in \cite{kuzelka2018relational} (which, for completeness, we reprove in the  appendix as Lemma 5%\ref{lemma:crazy1}%\nb{Why?}
).
For the second part, we may first notice that, after $\cC_\Upsilon$ is sampled and fixed, $Q_{\Upsilon,k}(f) = \mathbb{E}\left[ f\left(\Upsilon\langle\cS_{j,i}'\rangle \right) \right]$, as the probability of $\cS_{j,i}'$ being a particular size-$k$ subset of $\cC_\Upsilon$ is the same for all such subsets.
The second part can then be shown by applying Hoeffding inequality to $q$ i.i.d.\ samples $\frac{1}{\lfloor n/k \rfloor}\sum_{i=1}^{\lfloor n/k \rfloor}f(\Upsilon\langle\cS_{j,i}\rangle)$, $j=1,2,\ldots,q$, which have the same expected value
$Q_{\Upsilon,k}(f)$. \qed
\end{proof}

%\subsection{Are $\lfloor n/k \rfloor$ Samples Enough?}
%\noindent \nb{Removed subsection heading, as part of example from this subsection was moved to the introduction.}
\noindent At this point, one might wonder if the above lemma already gives us a way to find VC-type bounds for relational data, based on the following strategy: sample $\lfloor n/k \rfloor$ size-$k$ fragments from a given training relational example $\Upsilon$ using the procedure defined in Lemma \ref{lemma:many_vectors} and use this set of fragments as our training data. Although this would allow us to use standard bounds that are known for learning from i.i.d. data \cite{vapnik_book}, there are two problems with this approach. The first problem is that in reality we do not always know the size of the global example $\aleph$ and hence we do not know how to get a sample of $\lfloor n/k \rfloor$ size-$k$ sets that behaves as an independent sample from $\aleph$ (noting that we need to know the size of $\aleph$ to define the set $\mathcal{I}'$ in Lemma \ref{lemma:many_vectors}). The second problem is that there are cases where only sampling the $\lfloor n/k \rfloor$ samples is sub-optimal from the point of view of statistical power, as we illustrate in the next example.

%An intuitive approach would be to simulate generation of i.i.d.\ subsamples from the global structure by sampling from a suitably selected distribution on subsamples of the given training example \cite{kuzelka2018relational} and apply the classical bounds from statistical learning theory \cite{vapnik_book}. While this strategy is {sometimes} possible, it is sub-optimal from the point of view of statistical power, as is illustrated in the following example.

\begin{example}\label{ex1}
Consider a global structure which takes the form of a large directed graph, and assume that we are interested in estimating the probability that the formula $\exists X,Y : \textit{edge}(X,Y)$ holds for a fragment of the structure induced by two randomly sampled nodes. Assume furthermore that the given graph was generated by sampling (directed) edges independently with some probability $p$. The probability that $\exists X,Y : \textit{edge}(X,Y)$ holds for any two nodes will thus correspond to some value $p^*$ close to $1-(1-p)^2$. As we will see, given a training fragment induced by $n$ nodes from this graph, we can only generate $\lfloor \frac{n}{2} \rfloor$ samples that behave like i.i.d.\ samples. In this case, a more accurate estimate of $p^*$ can be obtained by using all size-$2$ fragments of the training fragment.
\end{example}

Nonetheless, the strategy based on sampling $\lfloor n/k \rfloor$ size-$k$ fragments may actually be optimal in the worst case as we illustrate in the next example.

\begin{example}
Let us again consider the setting from Example \ref{ex1}, which we can now describe more formally. In particular, assume that $\aleph = (\cA_\aleph,\cC_\aleph)$ represents a large directed graph. Let $k = 2$ and $\Phi = \{ \exists X,Y : \textit{edge}(X,Y) \}$.
%Hence, $Q_{\aleph,k}(\Phi)$ is equal to the fraction of pairs of nodes of the network that are connected by an edge.
Let $\Upsilon$ be a relational example sampled uniformly from $\aleph$ (i.e. $\Upsilon = \aleph\langle \cC_\Upsilon \rangle$ where $\cC_\Upsilon$ is sampled uniformly from all size-$n$ subsets of $\cC_\aleph$). Let us now, in contrast to the assumption underlying Example \ref{ex1}, assume that the directed graph was constructed using the following process.
%We may consider two extreme cases. In the first one, the network is constructed by sampling the edges independently with some probability $p$. Then $Q_{\aleph,k}(\Phi) = p^* \approx 2p-p^2$. Now, if we use all size-$2$ fragments of $\Upsilon$, we get more accurate estimate of $p^*$ than if we just used the $\lfloor n/k \rfloor$ samples. This is because of the way the edges of the network were sampled independently (note that it is important here that we assume $\aleph$ to be random and constructed by the above process).
%Let us look at the other extreme. Here, we assume that the network is constructed as follows.
For all nodes $v$, we flip a biased coin with probability of heads being $q$. If it lands heads, we add a directed edge from $v$ to all other nodes. In this case\footnote{More formally, the following holds, assuming $\aleph$ is generated by the respective random processes. In the setting from Example \ref{ex1} we have $\mathbb{E}_{\aleph}\left[Q_{\aleph,k}(\Phi)\right] = 1-(1-p)^2$ and in the setting from this example we have $\mathbb{E}_{\aleph}\left[Q_{\aleph,k}(\Phi)\right] = 1-(1-q)^2$.}, $Q_{\aleph,k}(\Phi) = p' \approx 1-(1-q)^2$. The main difference with the setting from Example \ref{ex1} is that estimating $p'$ now effectively corresponds to estimation of a property of nodes, as we are also able to recover $p'$ by observing how many nodes have at least one outgoing edge. However, this also means that the effective sample size in this case only grows linearly with the number of vertices (as opposed to quadratically in Example \ref{ex1}). This, at least asymptotically (up to a multiplicative constant), is a worst-case scenario as the number of independent samples that we are able to obtain using Lemma \ref{lemma:many_vectors} also grows linearly with the number of vertices in the sample $\Upsilon$ (i.e. linearly with $|\cC_\Upsilon|$).
\end{example}

%\yuyi{(Some results from ref.\ \cite{nandi1963properties} also support what we want to say here:
%1. no matter whether the edges are independent or fully node-dependent, u-statistics always have the smallest variance.
%2. in a large range, the variance is inversely proportional to $n/k$, if we ignore small terms. However, I am not sure it is good to mention these here, and point 2 might have been mentioned already in the AAAI paper. )}

% %It turns out that, in reality,
% \noindent \steven{To obtain suitable bounds, we will instead} need to estimate $Q_{\aleph,k}(f)$ using $Q_{\Upsilon,k}(f)$. %where $\Upsilon$ is a training example $\Upsilon$.
% However, this also means that we need to be able to bound the error of $Q_{\Upsilon,k}(f)$\nb{Isn't that what we were trying to do all along? The way this is formulated is confusing, as it appears to suggest that a new strategy/idea is introduced here.}, which in turn means that we cannot use existing bounds for classical i.i.d. learning. \steven{Our main contribution} is the derivation of such bounds.

\subsection{Bounding Expected Error}\label{sec:expected}

In this section we use the results from Section \ref{sec:extracting} to obtain a bound on the expected value of $\sup_{f\in H}\left|Q_{\aleph,k}(f)-Q_{\Upsilon,k}(f)\right|$.

\begin{lemma}\label{lemma:law_of_large_numbers}
Let $\aleph = (\cA_\aleph, \cC_\aleph)$ be a global example and $\cC_\Upsilon$ be sampled uniformly from all size-$n$ subsets of $\cC_\aleph$ and let us define $\Upsilon = \aleph \langle \cC_\Upsilon \rangle$. Let $\mathbf{Y}_j = (\cS_{j,1}',\dots,\cS_{j,\lfloor \frac{n}{k} \rfloor}')$, where $1 \leq j \leq q$, be random vectors sampled as in Lemma \ref{lemma:many_vectors}. Then for any hypothesis class $\mathcal{H}$ of functions $f : \Omega(\mathcal{L},k) \rightarrow \{0,1\}$ with finite VC-dimension $d$, the following holds:
{\small \begin{align*}
\mathbb{E}\left[\sup_{f\in \mathcal{H}}\left|Q_{\aleph,k}(f)-Q_{\Upsilon,k}(f)\right|\right] \leq \lim_{q\rightarrow \infty} \mathbb{E}\left[\sup_{f\in \mathcal{H}}\left|Q_{\aleph,k}(f)-\frac{1}{q \cdot \left\lfloor \frac{n}{k} \right\rfloor}\sum_{i=1}^q \sum_{\cS \in \mathbf{Y}_i} f(\Upsilon\langle\cS\rangle) \right|\right]
\end{align*}}
\end{lemma}
\begin{proof}
%\nb{Removed: ``First, it follows from Remark \ref{remark:finitely_many_hypotheses} that the supremum needs to be taken only over a finite number $t$ of hypotheses, one from each equivalence class of functions that are equal on all size-$k$ subsets of $\cC_\aleph$'' as it is only used further on}.
We have
%  \begin{eqnarray*}
% \lim_{q\rightarrow \infty} \mathbb{E}\left[\sup_{f\in \mathcal{H}}\left|Q_{\aleph,k}(f)-Q_{\Upsilon,k}(f)\right|\right] \\
% = \lim_{q\rightarrow \infty}
% \mathbb{E}\left[\sup_{f\in \mathcal{H}}\left|Q_{\aleph,k}(f)-\left(\frac{1}{q \left\lfloor \frac{n}{k} \right\rfloor}\sum_{i=1}^q \sum_{\cS \in \mathbf{Y}_i} f(\Upsilon\langle\cS\rangle) \right) \right.\right. \\
% \left.\left. + \left(\frac{1}{q \left\lfloor \frac{n}{k} \right\rfloor}\sum_{i=1}^q \sum_{\cS \in \mathbf{Y}_i} f(\Upsilon\langle\cS\rangle) \right) - Q_{\Upsilon,k}(f) \right|\right] \\
% \leq \lim_{q\rightarrow \infty}
% \mathbb{E}\left[\sup_{f\in \mathcal{H}}\left|Q_{\aleph,k}(f)-\left(\frac{1}{q \left\lfloor \frac{n}{k} \right\rfloor}\sum_{i=1}^q \sum_{\cS \in \mathbf{Y}_i} f(\Upsilon\langle\cS\rangle) \right) \right|\right] \\
% + \lim_{q\rightarrow \infty} \mathbb{E}\left[\sup_{f\in \mathcal{H}}\left|\left(\frac{1}{q \left\lfloor \frac{n}{k} \right\rfloor}\sum_{i=1}^q \sum_{\cS \in \mathbf{Y}_i} f(\Upsilon\langle\cS\rangle) \right) - Q_{\Upsilon,k}(f) \right|\right]
% % \text{(Using Lemma \ref{lemma:many_vectors}, Remark \ref{remark:finitely_many_hypotheses} and union bound over the functions' equivalence classes)}\\
% % = \mathbb{E}\left[\sup_{f\in \mathcal{H}}\left|Q_{\aleph,k}(f)-\frac{1}{q \left\lfloor \frac{n}{k} \right\rfloor}\sum_{i=1}^q \sum_{\cS \in \mathbf{Y}_i} f(\Upsilon\langle\cS\rangle) \right|\right]
% \end{eqnarray*}
%\nb{Aligned the formulas}
 \begin{align}
&\mathbb{E}\left[\sup_{f\in \mathcal{H}}\left|Q_{\aleph,k}(f)-Q_{\Upsilon,k}(f)\right|\right] \notag\\
&\quad\quad= \lim_{q\rightarrow \infty}
\mathbb{E}\left[\sup_{f\in \mathcal{H}}\left|Q_{\aleph,k}(f)-\left(\frac{1}{q \left\lfloor \frac{n}{k} \right\rfloor}\sum_{i=1}^q \sum_{\cS \in \mathbf{Y}_i} f(\Upsilon\langle\cS\rangle) \right) \right.\right.  \notag\\
&\quad\quad\quad\quad\quad\quad\left.\left. + \left(\frac{1}{q \left\lfloor \frac{n}{k} \right\rfloor}\sum_{i=1}^q \sum_{\cS \in \mathbf{Y}_i} f(\Upsilon\langle\cS\rangle) \right) - Q_{\Upsilon,k}(f) \right|\right]  \notag\\
&\quad\quad\leq \lim_{q\rightarrow \infty}
\mathbb{E}\left[\sup_{f\in \mathcal{H}}\left|Q_{\aleph,k}(f)-\left(\frac{1}{q \left\lfloor \frac{n}{k} \right\rfloor}\sum_{i=1}^q \sum_{\cS \in \mathbf{Y}_i} f(\Upsilon\langle\cS\rangle) \right) \right|\right]  \notag\\
&\quad\quad\quad\quad\quad\quad+ \lim_{q\rightarrow \infty} \mathbb{E}\left[\sup_{f\in \mathcal{H}}\left|\left(\frac{1}{q \left\lfloor \frac{n}{k} \right\rfloor}\sum_{i=1}^q \sum_{\cS \in \mathbf{Y}_i} f(\Upsilon\langle\cS\rangle) \right) - Q_{\Upsilon,k}(f) \right|\right]\label{last_line}
% \text{(Using Lemma \ref{lemma:many_vectors}, Remark \ref{remark:finitely_many_hypotheses} and union bound over the functions' equivalence classes)}\\
% = \mathbb{E}\left[\sup_{f\in \mathcal{H}}\left|Q_{\aleph,k}(f)-\frac{1}{q \left\lfloor \frac{n}{k} \right\rfloor}\sum_{i=1}^q \sum_{\cS \in \mathbf{Y}_i} f(\Upsilon\langle\cS\rangle) \right|\right]
\end{align}
To finish the proof, we show that the last summand in (\ref{last_line}) is zero. To this end, first note that it follows from Remark \ref{remark:finitely_many_hypotheses} that the supremum only needs to be taken over a finite number $t$ of hypotheses, one from each equivalence class of functions that are equal on all size-$k$ subsets of $\cC_\aleph$. Together with  Lemma \ref{lemma:many_vectors} and the union bound on the finitely many equivalence classes, we find
\begin{align*}
    P\left[\sup_{f\in \mathcal{H}}\left|\left(\frac{1}{q \left\lfloor \frac{n}{k} \right\rfloor}\sum_{i=1}^q \sum_{\cS \in \mathbf{Y}_i} f(\Upsilon\langle\cS\rangle) \right) - Q_{\Upsilon,k}(f) \right| \geq \varepsilon \right] \leq 2 \cdot t \cdot \exp{\left( - 2q \varepsilon^2 \right)}
\end{align*}
%where $t$ is the number of equivalence classes of functions from $\mathcal{F}$\nb{What is $\mathcal{F}$? I assume you mean that $t$ is the number of hypothesis equivalence classes?} on $\Upsilon$.
Then it follows using $\expect{X} = \int_{0}^1 P[X \geq x] dx$ (assuming $P[X \in [0;1]] = 1$) that
\begin{multline*}
    \mathbb{E}\left[\sup_{f\in \mathcal{H}}\left|\left(\frac{1}{q \left\lfloor \frac{n}{k} \right\rfloor}\sum_{i=1}^q \sum_{\cS \in \mathbf{Y}_i} f(\Upsilon\langle\cS\rangle) \right) - Q_{\Upsilon,k}(f) \right|\right] \leq \int_{0}^1 2 \cdot t \cdot \exp{\left( - 2q x^2 \right)} dx.
\end{multline*}
Finally, noticing that $\lim_{q \rightarrow \infty}\int_{0}^1 2 \cdot t \cdot \exp{\left( - 2q x^2 \right)} dx = 0$ finishes the proof. \qed
\end{proof}

\begin{lemma}\label{lemma:classical_vc}
Suppose $\mathbf{Y}_j = (\cS_{j,1}',\dots,\cS_{j,\lfloor \frac{n}{k} \rfloor}')$ is a random vector sampled as in Lemma \ref{lemma:many_vectors}. Then for any hypothesis class of functions $f : \Omega(\mathcal{L},k) \rightarrow \{0,1\}$ with VC-dimension $d$ we have: %\todo{Let's make it directly refer to VC-dim.}
% \[P\left[\sup_{f \in \cal{H}} \left| Q_{\aleph,k}(f) - \frac{1}{\lfloor n/k \rfloor} \sum_{\cS \in \mathbf{Y}} f(\Upsilon\langle \cS \rangle)  \right| \ge \varepsilon \right] \le 4 \mathcal{S}(\mathcal{H},2\lfloor n/k \rfloor) \cdot \exp\left(-\frac{\lfloor n/k \rfloor\varepsilon^2}{8}\right)\]
\[P\left[\sup_{f \in \cal{H}} \left| Q_{\aleph,k}(f) - \frac{1}{\lfloor n/k \rfloor} \sum_{\cS \in \mathbf{Y}_j} f(\Upsilon\langle \cS \rangle)  \right| \ge \varepsilon \right] \le 4 \left(\frac{2e\lfloor n/k \rfloor}{d}\right)^d \exp\left(-\frac{\lfloor n/k \rfloor\varepsilon^2}{8}\right)\]
and %\todo{yuyi: I need to check whether this is the best known expected value bound}
% \[\mathbb{E}\left[\sup_{f \in \cal{H}} \left| Q_{\aleph,k}(f) - \frac{1}{\lfloor n/k \rfloor} \sum_{\cS \in \mathbf{Y}_j} f(\Upsilon\langle \cS \rangle)  \right|\right] \le 2 \sqrt{\frac{\log \mathcal{S}(\mathcal{H},\lfloor n/k \rfloor) + \log 2}{\lfloor n/k \rfloor}}\]
\[\mathbb{E}\left[\sup_{f \in \cal{H}} \left| Q_{\aleph,k}(f) - \frac{1}{\lfloor n/k \rfloor} \sum_{\cS \in \mathbf{Y}_j} f(\Upsilon\langle \cS \rangle)  \right|\right] \le 2 \sqrt{\frac{2d \log \left(2 e \lfloor n/k \rfloor / d \right)}{\lfloor n/k \rfloor}}\]
\end{lemma}
\begin{proof}
Since $\cS_{j,1}',\dots,\cS_{j,\lfloor \frac{n}{k} \rfloor}'$ are sampled in an i.i.d.\ way, the classical VC inequality applies \cite{vapnik1971uniform}. The expected value bound can be derived from the bound (6.4) in \cite{shalev2014understanding}\footnote{The specific form that we use here can be found in the lecture notes of Philippe Rigollet {\tt https://bit.ly/2H89wPn}.}. \qed
\end{proof}

\noindent We are now ready to prove Theorem \ref{lemma:expected_error}.

% \begin{theorem}\label{lemma:expected_error}
% Let $\aleph = (\cA_\aleph, \cC_\aleph)$ be a global example and $\cC_\Upsilon$ be sampled uniformly from all size-$n$ subsets of $\cC_\aleph$ and let us define $\Upsilon = \aleph \langle \cC_\Upsilon \rangle$. Then for any hypothesis class $\mathcal{H}$ of function $f : \Omega(\mathcal{L},k) \rightarrow \{0,1\}$ with finite VC-dimension $d$, the following holds:
% \begin{align*}
% \mathbb{E}\left[\sup_{f\in \mathcal{H}}\left|Q_{\aleph,k}(f)-Q_{\Upsilon,k}(f)\right|\right] \leq  2 \sqrt{\frac{\log \mathcal{S}(\mathcal{H},\lfloor n/k \rfloor) + \log 2}{\lfloor n/k \rfloor}}
% \end{align*}
% \end{theorem}
\begin{proof}[of Theorem \ref{lemma:expected_error}]
Let $\mathbf{Y}_j = (\cS_{j,1}',\dots,\cS_{j,\lfloor \frac{n}{k} \rfloor}')$, where $1 \leq j \leq q$ for a given integer $q$, be random vectors sampled as in Lemma \ref{lemma:many_vectors}. %For convenience, let us also define $$R_{\Upsilon}^{(q)}(f) = \frac{1}{q} \sum_{j=1}^q \frac{1}{\lfloor n/k \rfloor} \sum_{\mathcal{S} \in \mathbf{Y}_j} f(\Upsilon\langle\cS\rangle).$$
First, using Lemma \ref{lemma:law_of_large_numbers} for the first step, we find
\begin{align}
& \mathbb{E}\left[\sup_{f\in H}\left|Q_{\aleph,k}(f)-Q_{\Upsilon,k}(f)\right|\right] \notag\\
&\quad\quad \le \lim_{q\rightarrow \infty} \mathbb{E}\left[\sup_{f\in H}\left|Q_{\aleph,k}(f)-\frac{1}{q} \sum_{j=1}^q \frac{1}{\lfloor n/k \rfloor} \sum_{\mathcal{S} \in \mathbf{Y}_j} f(\Upsilon\langle\cS\rangle) \right|\right] \notag\\ %\label{multline:1}
&\quad\quad = \lim_{q\rightarrow \infty} \mathbb{E}\left[ \sup_{f\in H}\left|\frac{1}{q}\sum_{j=1}^q\left( Q_{\aleph,k}(f) -  \frac{1}{\lfloor n/k \rfloor} \sum_{\mathcal{S} \in \mathbf{Y}_j} f(\Upsilon\langle\cS\rangle) \right) \right|\right] \notag\\
&\quad\quad\le \lim_{q\rightarrow \infty} \mathbb{E}\left[\frac{1}{q}\sup_{f\in H}\sum_{j=1}^q\left|Q_{\aleph,k}(f)-  \frac{1}{\lfloor n/k \rfloor} \sum_{\mathcal{S} \in \mathbf{Y}_j} f(\Upsilon\langle\cS\rangle) \right|\right]\notag\\
&\quad\quad\le \lim_{q\rightarrow \infty} \mathbb{E}\left[\frac{1}{q}\sum_{j=1}^q\sup_{f\in H}\left|Q_{\aleph,k}(f)-  \frac{1}{\lfloor n/k \rfloor} \sum_{\mathcal{S} \in \mathbf{Y}_j} f(\Upsilon\langle\cS\rangle) \right|\right]\notag\\
&\quad\quad= \lim_{q\rightarrow \infty} \frac{1}{q} \sum_{j=1}^q \mathbb{E}\left[ \sup_{f\in H} \left|Q_{\aleph,k}(f)-  \frac{1}{\lfloor n/k \rfloor} \sum_{\mathcal{S} \in \mathbf{Y}_j} f(\Upsilon\langle\cS\rangle) \right|\right] \notag\\
&\quad\quad= \mathbb{E}\left[ \sup_{f\in H} \left|Q_{\aleph,k}(f)-  \frac{1}{\lfloor n/k \rfloor} \sum_{\mathcal{S} \in \mathbf{Y}_1} f(\Upsilon\langle\cS\rangle) \right|\right] \label{multline:2}
% \le \mathbb{E}\left[\frac{1}{q}\sum_{i=1}^q\sup_{f\in H}\left|Q_{\aleph,k}(f)-Q_{S^{(i)},k}(f)\right|\right]
% = \mathbb{E}\left[\sup_{f\in H}\left|Q_{\aleph,k}(f)-Q_{S,k}(f)\right|\right] \\
%  \text{(using Lemma \ref{lemma:classical_vc})} \\
% \le 2 \sqrt{\frac{\log \mathcal{S}(\mathcal{H,\mathsf{s}}) + \log 2}{\mathsf{s}}}
\end{align}
Note that the last equality is a consequence of Lemma \ref{lemma:many_vectors}, from which it among others follows that all $\mathbf{Y}_j$'s have the same distribution. In other words, all the $q$ expected values are equal. Finally, we can use Lemma \ref{lemma:classical_vc} to bound (\ref{multline:2}) which finishes the proof.
\qed
\end{proof}

\noindent It is also possible to get rid of the logarithmic factor in the bound on expected error. However, as mentioned in \cite{devroye2013probabilistic}, such bounds are worse up to very large training set sizes due to the increased constant factors.

\subsection{From Tail Bounds to Moment-Generating Functions}\label{sec:characteristic_from_tail}

In this section, we derive bounds on the moment-generating function of a random variable from its tail bounds.

\begin{lemma}\label{lemma:characteristic}
For a non-negative random variable $X$, if there exist constants $C\ge e$ and $B> 0$ such that
$$P[X\ge t] \le C \exp(-t^2/B)\qquad \forall t \ge 0,$$
then for any $\lambda > 0$
$$\expect{\exp{\left( \lambda X \right)}} \leq 1 + \lambda C \sqrt{\pi B} \exp{\left(\frac{\lambda^2 B}{4} \right)}$$
\end{lemma}
\begin{proof}
We have:
\begin{align*}
	\mathbb{E}\left[ X^{p} \right]
	&=
	\int_{0}^{\infty} \mathrm{P}\left( {X}^{p} \ge u\right) du
	=
	\int_{0}^{\infty} \mathrm{P}\left( {X}^{p} \ge t^{p}\right) \cdot {p} \cdot t^{p-1} dt
	\\
	&=
	\int_{0}^{\infty} \mathrm{P}\left( {X} \ge t\right) \cdot {p} \cdot {t^{p-1}} dt
	\le
	\int_{0}^{\infty}
		C \cdot e^{-t^{2}/B} \cdot {p} \cdot {t^{p-1}} dt
\end{align*}
Next, for the moment-generating function, we have%\nb{Reformatted/aligned formulas}
\begin{align*}
    \expect{\exp{\left( \lambda X \right)}} &\leq 1  + \sum_{p = 1}^\infty \frac{\lambda^p \expect{X^p}}{p!} \leq
    1  + \sum_{p = 1}^\infty \frac{\lambda^p \int_{0}^{\infty} C \cdot e^{-t^{2}/B} \cdot {p} \cdot {t^{p-1}} dt}{p!} \\
    &\leq 1  + C \int_{0}^{\infty} e^{-t^{2}/B} \cdot \sum_{p = 1}^\infty \frac{\lambda^p   {p} \cdot {t^{p-1}} }{p!} dt \\
    &= 1  + C \lambda \int_{0}^{\infty} e^{-t^{2}/B} \cdot \sum_{p = 0}^\infty \frac{\lambda^p  \cdot {t^{p}} }{p!} dt \\
    &= 1  + C \lambda \int_{0}^{\infty} e^{-t^{2}/B} \cdot e^{t \lambda} dt
    = 1  + C \lambda \int_{0}^{\infty} e^{-\frac{\left(t - \frac{1}{2} \lambda B \right)^2}{B} + \frac{\lambda^2 B^2}{4}} dt \\
    &= 1  + C \lambda e^{\frac{\lambda^2 B^2}{4} } \int_{0}^{\infty} e^{-\frac{\left(t - \frac{1}{2} \lambda B \right)^2}{B}} dt\\
    &= 1  + \frac{1}{2} C \lambda  \sqrt{\pi B}  e^{\frac{\lambda^2 B^2}{4} } \left( \operatorname{erf}\left( \frac{\lambda \sqrt{B}}{2} \right)  + 1 \right) \\
    &\leq 1 + C \lambda \sqrt{\pi B} \exp{\left(\frac{\lambda^2 B}{4} \right)}
\end{align*}
Note that it is easy to check that all the series in the above derivation converge absolutely. {The} Fubini-Tonelli theorem justifies the change of order of summation and integration. %We then obtain $\expect{\exp{\left( \lambda X \right)}} \leq C \cdot \left(1 + (2e)^{1/4} \right) \cdot \exp{\left(\lambda^2 B \right)} + \lambda \expect{X},$ which finishes the proof of the lemma.
\qed
\end{proof}

\subsection{From Moment-Generating Functions to Tail Bounds}\label{sec:tail_from_characteristic}

We can now finish the proof of our main result, Theorem \ref{thm:thm_main}.

\begin{proof}[of Theorem \ref{thm:thm_main}]
Let $\mathbf{Y}_j = (\cS_{j,1}',\dots,\cS_{j,\lfloor \frac{n}{k} \rfloor}')$, for $1 \leq j \leq q$, be random vectors sampled as in Lemma \ref{lemma:many_vectors}. For convenience, let us also define
$$R_{\Upsilon}^{(q)}(f) = \frac{1}{q} \sum_{j=1}^q \frac{1}{\lfloor n/k \rfloor} \sum_{\mathcal{S} \in \mathbf{Y}_j} f(\Upsilon\langle\cS\rangle).$$
First, we have
\begin{multline*}
P\left[\sup_{f \in \cal{H}} \left| Q_{\aleph,k}(f) - Q_{\Upsilon,k}(f)  \right| \ge \varepsilon \right] \\
= P\left[\sup_{f \in \cal{H}} \left\{ \left| Q_{\aleph,k}(f) - R_{\Upsilon}^{(q)}(f) + R_{\Upsilon}^{(q)}(f) -  Q_{\Upsilon,k}(f) \right|  \right\} \ge \varepsilon \right]\\
\leq P\left[\sup_{f \in \cal{H}} \left\{ \left|  Q_{\aleph,k}(f) - R_{\Upsilon}^{(q)}(f)\right| + \left|R_{\Upsilon}^{(q)}(f) -  Q_{\Upsilon,k}(f)   \right| \right\} \ge \varepsilon \right] \\
\leq P\left[\sup_{f \in \cal{H}} \left\{ \left|  Q_{\aleph,k}(f) - R_{\Upsilon}^{(q)}(f)\right|\right\} + \sup_{f \in \cal{H}} \left\{ \left|R_{\Upsilon}^{(q)}(f) -  Q_{\Upsilon,k}(f)   \right| \right\} \ge \varepsilon \right]
\end{multline*}
It follows from the fact that the supremum needs to be taken only over the finitely many equivalence classed of $\mathcal{H}$ on $\aleph$ and from Lemma \ref{lemma:many_vectors} (see the discussion in the proof of Lemma 5) that for any $\varepsilon^* > 0$ and $\delta^* > 0$ there is an integer $q_0$ such that for all $q \geq q_0$:
$$P\left[\sup_{f \in \cal{H}} \left\{ \left|R_{\Upsilon}^{(q)}(f) -  Q_{\Upsilon,k}(f)   \right| \right\} \ge \varepsilon^* \right] \leq \delta^*.$$
Hence, for any $\varepsilon^* > 0$, $\delta^* > 0$ and a suitably large $q \geq q_0$ we have
\begin{multline*}
P\left[\sup_{f \in \cal{H}} \left\{ \left|  Q_{\aleph,k}(f) - R_{\Upsilon}^{(q)}(f)\right|\right\} + \sup_{f \in \cal{H}} \left\{ \left|R_{\Upsilon}^{(q)}(f) -  Q_{\Upsilon,k}(f)   \right| \right\} \ge \varepsilon \right] \\
\leq P\left[\sup_{f \in \cal{H}} \left\{ \left|  Q_{\aleph,k}(f) - R_{\Upsilon}^{(q)}(f)\right|\right\} \ge \varepsilon-\varepsilon^* \right] + \delta^*.
\end{multline*}
Taking the limit $q_0 \rightarrow \infty$ we obtain
\begin{multline*}
P\left[\sup_{f\in \mathcal{H}}\left|Q_{\aleph,k}(f)-Q_{\Upsilon,k}(f)\right| \geq \varepsilon \right]
\leq \lim_{q\rightarrow \infty} P\left[\sup_{f\in \mathcal{H}}\left|Q_{\aleph,k}(f)-R^{(q)}_{\Upsilon}(f) \right| \geq \varepsilon \right]
\end{multline*}
Next we need to bound the right-hand side of the above inequality. For any $q$ we have
\begin{align*}
    &P\left[\sup_{f\in \mathcal{H}}\left|Q_{\aleph,k}(f)-R^{(q)}_{\Upsilon}(f) \right| \geq \varepsilon \right] \\
    &\quad\quad\quad\quad = P\left[\sup_{f \in \mathcal{H}}\left|Q_{\aleph,k}(f) -\frac{1}{q} \sum_{j=1}^q \frac{1}{\lfloor n/k \rfloor} \sum_{\mathcal{S} \in \mathbf{Y}_j} f(\Upsilon\langle\cS\rangle) \right|  \geq \varepsilon \right] \\
    &\quad\quad\quad\quad= P\left[\sup_{f \in \mathcal{H}}\left|\frac{1}{q} \sum_{j=1}^q \left(Q_{\aleph,k}(f)  - \frac{1}{\lfloor n/k \rfloor} \sum_{\mathcal{S} \in \mathbf{Y}_j} f(\Upsilon\langle\cS\rangle) \right) \right|  \geq \varepsilon \right] \\
    &\quad\quad\quad\quad \leq P\left[\sup_{f \in \mathcal{H}}\frac{1}{q} \sum_{j=1}^q \left| Q_{\aleph,k}(f)  - \frac{1}{\lfloor n/k \rfloor} \sum_{\mathcal{S} \in \mathbf{Y}_j} f(\Upsilon\langle\cS\rangle) \right| \geq \varepsilon \right] \\
    &\quad\quad\quad\quad \leq P\left[\frac{1}{q} \sum_{j=1}^q \sup_{f \in \mathcal{H}} \left| Q_{\aleph,k}(f)  - \frac{1}{\lfloor n/k \rfloor} \sum_{\mathcal{S} \in \mathbf{Y}_j} f(\Upsilon\langle\cS\rangle) \right| \geq \varepsilon \right]
\end{align*}
Let us denote
$$T_j = \sup_{f \in \mathcal{H}} \left| Q_{\aleph,k}(f)  - \frac{1}{\lfloor n/k \rfloor} \sum_{\mathcal{S} \in \mathbf{Y}_j} f(\Upsilon\langle\cS\rangle) \right|.$$
Combining Lemma \ref{lemma:classical_vc} and Lemma \ref{lemma:characteristic}, we can bound $\expect{\exp{\left(\lambda T_j\right)} }$ as
%$$\expect{\exp{\left( \lambda T_j \right)}} \leq 4 \left(\frac{2e\lfloor n/k \rfloor}{d}\right)^d \cdot \left(1 + (2e)^{1/4} \right) \cdot \exp{\left( \frac{8 }{\lfloor n/k \rfloor} \lambda^2 e^2 \right)} + \lambda \expect{T_j}.$$
\begin{align*}
    \expect{\exp{\left( \lambda T_j \right)}} \leq 1 + 4 \lambda \sqrt{\frac{8 \pi}{\lfloor n/k \rfloor}} \left(\frac{2e \lfloor n/k \rfloor}{d} \right)^d \exp{\left( \frac{2 \lambda^2}{\lfloor n/k \rfloor} \right)}.
\end{align*}
Let us denote $T = \frac{1}{q} \sum_{j=1}^q T_j$. We use the observation from \cite{hoeffding1963probability} that due to Jensen's inequality and linearity of expectation
\begin{align*}
    \expect{\exp{(\lambda T)}} \leq \frac{1}{q} \sum_{j=1}^q \expect{\exp{(\lambda T_j)}} = \expect{\exp{(\lambda T_1)}}.
\end{align*}
Next we obtain a bound on $P[T \geq \varepsilon]$ from the bound on $\expect{\exp{(\lambda T)}}$. In particular, for positive $\lambda$, we have
\begin{align*}
    P[T \geq \varepsilon] &= P[e^{\lambda \cdot X} \geq e^{\lambda \cdot \varepsilon}]
    \leq e^{-\lambda \cdot \varepsilon} \expect{e^{\lambda \cdot T}} \\
    &\leq e^{-\lambda \cdot \varepsilon} \left( 1 + 4 \lambda \sqrt{\frac{8 \pi}{\lfloor n/k \rfloor}} \left(\frac{2e \lfloor n/k \rfloor}{d} \right)^d \exp{\left( \frac{2 \lambda^2}{\lfloor n/k \rfloor} \right)} \right).
\end{align*}
where the Markov inequality was used for the third step.
Since the above bound holds for any $q$, it also holds in the limit.
Next, we can plug in $\lambda := \frac{\varepsilon \cdot \lfloor n/k \rfloor}{4}$ and obtain:
\begin{align*}
    P[T \geq \varepsilon]
    \leq \exp{\left(- \frac{\lfloor n/k \rfloor \varepsilon^2}{4}\right)} + \varepsilon \sqrt{8 \pi \lfloor n/k \rfloor} \left(\frac{2e \lfloor n/k \rfloor}{d} \right)^d \cdot \exp{\left( - \frac{\lfloor n/k \rfloor \varepsilon^2}{8} \right)}.
\end{align*}
\qed
\end{proof}

% \subsection{U-statistics}
% In Section \ref{sec:proofs}, we derived generalization error bounds for relational learning.

% related work:
% \cite{clemenccon2008ranking}: not generalization error bound, only kernels of order two, using a special noise condition

% \begin{lemma}

% \end{lemma}

\section{Concluding Remarks}

We have derived VC-dimension based bounds which can be applied in relational learning settings where one may assume that the training data ({i.e.\ some given} relational structure) was obtained from {a} larger relational structure by sampling without replacement. {This includes many of the typical application settings in which, for instance, Markov logic networks are used. The considered bounds are useful, among others, for the analysis of} relational marginal problems \cite{kuzelka2018relational} and PAC-reasoning in relational domains \cite{kuzelka2018arxivPACReasoning}.

There are several interesting {avenues for future work}. First, in this paper, we have not studied the realizable learning case for which, at least in the classical i.i.d. case, one can obtain faster convergence rates. It would be interesting to extend our results into the realizable case. {Similarly, it would be of interest to study} bounds under low-noise conditions \cite{clemenccon2008ranking}, {which} sit somewhere between the realizable case and the case studied in this paper. {Another} natural direction for future work would be to extend the PAC-Bayesian setting into relational learning, {as the bounds that are derived in this setting tend to be tighter in practice \cite{langford2003pac}}.

\bibliographystyle{splncs04}
\bibliography{ref}

%\ifappendix

\appendix

\section{Omitted Proofs}

\begin{lemma}[Ku\v{z}elka et al. \cite{kuzelka2018relational}]\label{lemma:crazy1} %\nb{Not exactly formulated in the same way as Lemma \ref{lemma:many_vectors}.}
Let $\aleph = (\cA_\aleph, \cC_\aleph)$ be a global example. Let $0 \leq n \leq |\cC_\aleph|$ and $0 \leq k \leq n$ be integers. Let $\mathbf{X} = (\cS_1,\cS_2,\dots,\cS_{\lfloor \frac{n}{k} \rfloor})$ be a vector of subsets of $\cC_\aleph$, each sampled uniformly and independently of the others from all size-$k$ subsets of $\cC_\aleph$. Next let $\mathcal{I}' = \{ 1,2,\dots,|\cC_\aleph| \}$ and let $\mathbf{Y} = (\cS_1',\cS_2',\dots,\cS_{\lfloor \frac{n}{k} \rfloor}')$ be a vector sampled by the following process:
\begin{enumerate}
    \item Sample $\mathcal{C}_\Upsilon$ uniformly from all size-$n$ subsets of $\mathcal{C}_\aleph$.
    \item Sample subsets $\mathcal{I}_1',\dots,\mathcal{I}_{\lfloor \frac{n}{k} \rfloor}'$ of size $k$ from $\mathcal{I}'$.
    \item Sample an injective function $g : \bigcup_{i=1}^{\lfloor n/k \rfloor} \mathcal{I}_i' \rightarrow \cC_\Upsilon$ uniformly from all such functions.
    \item Define $\cS_i' = g(\mathcal{I}_i')$ for all $0 \leq i \leq \lfloor \frac{n}{k} \rfloor$.
\end{enumerate}
Then $\mathbf{X}$ and $\mathbf{Y}$ have the same distribution.
\end{lemma}
% \noindent {\bf Lemma \ref{lemma:crazy1} (Ku\v{z}elka et al. \cite{kuzelka2018relational})}
% {\em Let $\aleph = (\cA_\aleph, \cC_\aleph)$ be a global example. Let $0 \leq n \leq |\cC_\aleph|$ and $0 \leq k \leq n$ be integers. Let $\mathbf{X} = (\cS_1,\cS_2,\dots,\cS_{\lfloor \frac{n}{k} \rfloor})$ be a vector of subsets of $\cC_\aleph$, each sampled uniformly and independently of the others from all size-$k$ subsets of $\cC_\aleph$. Next let $\mathcal{I}' = \{ 1,2,\dots,|\cC_\aleph| \}$ and let $\mathbf{Y} = (\cS_1',\cS_2',\dots,\cS_{\lfloor \frac{n}{k} \rfloor}')$ be a vector sampled by the following process:
% \begin{enumerate}
%     \item Sample $\mathcal{C}_\Upsilon$ uniformly from all size-$n$ subsets of $\mathcal{C}_\aleph$.
%     \item Sample subsets $\mathcal{I}_1',\dots,\mathcal{I}_{\lfloor \frac{n}{k} \rfloor}'$ of size $k$ from $\mathcal{I}'$.
%     \item Sample an injective function $g : \bigcup_{i=1}^{\lfloor n/k \rfloor} \mathcal{I}_i' \rightarrow \cC_\Upsilon$ uniformly from all such functions.
%     \item Define $\cS_i' = g(\mathcal{I}_i')$ for all $0 \leq i \leq \lfloor \frac{n}{k} \rfloor$.
% \end{enumerate}
% Then $\mathbf{X}$ and $\mathbf{Y}$ have the same distribution.}
\begin{proof}
For convenience, let us define some additional notation. We set $t = \lfloor n/k \rfloor$ and we define $\textit{Inj}(\mathcal{J},\mathcal{D})$ to be the set of all injective functions from $\mathcal{J}$ to $\mathcal{D}$. We also set $r = \left| \bigcup_{i=1}^{t} \mathcal{S}_i \right|$. For a set $\mathcal{B}$, $[\mathcal{B}]^n$ will denote the set of all size-$n$ subsets of $\mathcal{B}$. Then we have
{\small
\begin{align*}
   P[\mathbf{Y} = (S_1,\dots,S_t)] =  \sum_{C_\Upsilon \in [\cC_\aleph]^n} \frac{1}{|[\cC_\aleph]]^n|} \sum_{I_1 \in [\mathcal{I}]^k}\dots\sum_{I_t \in [\mathcal{I}]^k} \left(\frac{1}{|[\mathcal{I}]^k|}\right)^t \\
   \cdot
    \sum_{h \in \textit{Inj}(\bigcup_{j=1}^t I_j, C_\Upsilon)} \frac{1}{|\textit{Inj}(\bigcup_{j=1}^t I_j, c_\Upsilon)|}
    \cdot \mathds{1}\left( h(I_1) = S_1 \wedge \dots \wedge h(I_t) = S_t \right) \\
    = \frac{1}{|[\cC_\aleph]^n| |[\mathcal{I}]^k|^t} \sum_{C_\Upsilon \in [\cC_\aleph]]^n}  \sum_{I_1 \in [\mathcal{I}]^k} \dots \sum_{I_t \in [\mathcal{I}]^k}
    \sum_{h \in \textit{Inj}(\bigcup_{j=1}^t I_j, C_\Upsilon)} \frac{\mathds{1}\left( \bigwedge_{j=1}^{t} I_j = h^{-1}(S_j) \right)}{|\textit{Inj}(\bigcup_{j=1}^t I_j, C_\Upsilon)|}.
\end{align*}}
\noindent Next we need to analyze the above expression. First, we may notice that whenever
$$\mathds{1}\left( \bigwedge_{j=1}^{t} i_j = h^{-1}(s_j) \right) = 1,$$
we will have $\left| \textit{Inj}(\bigcup_{j=1}^t I_j, C_\Upsilon) \right| = \left( \begin{array}{c} n \\ r \end{array} \right) r!$.

Next we need to evaluate the sum
\begin{align*}
\sum_{C_\Upsilon \in [\cC_\aleph]]^n}  \sum_{I_1 \in [\mathcal{I}]^k} \dots \sum_{I_t \in [\mathcal{I}]^k}
    \sum_{h \in \textit{Inj}(\bigcup_{j=1}^t I_j, C_\Upsilon)} \mathds{1}\left( \bigwedge_{j=1}^{t} I_j = h^{-1}(S_j) \right).
\end{align*}
If the sum
\begin{align*}
\sum_{I_1 \in [\mathcal{I}]^k} \dots \sum_{I_t \in [\mathcal{I}]^k}
    \sum_{h \in \textit{Inj}(\bigcup_{j=1}^t I_j, C_\Upsilon)} \mathds{1}\left( \bigwedge_{j=1}^{t} I_j = h^{-1}(S_j) \right).
\end{align*}
is non-zero for a $\cC_\Upsilon$ then it must also be true for any $\cC_\Upsilon'$ obtained by replacing elements in $\cC_\Upsilon \setminus \cup_{i=1}^{t} S_i$ by some other elements of $\cC_\aleph$; this will give us a factor $$\left( \begin{array}{c} |\cC_\aleph|-r \\ n-r \end{array} \right).$$

For a given $(I_1,I_2,\dots,I_t)$, there is exactly one
$$h \in \textit{Inj}(\bigcup_{j=1}^t I_j, C_\Upsilon)$$
such that
$$\mathds{1}( \bigwedge_{j=1}^{t} i_j = h^{-1}(s_j)) = 1$$
whenever $(I_1,I_2,\dots,I_t)$ is isomorphic to $(S_1,S_2,\dots,S_t)$, i.e. whenever there is a bijection between $\bigcup_{j=1}^t I_j$ and $\bigcup_{j=1}^t S_j$ that preserves the set structure. Hence, to obtain all such $(I_1',I_2',\dots,I_t')$ we can take all injective functions from one fixed $\cup_{j=1}^t I_j$ to $\mathcal{I}$ and apply them on it. There are $$\left(\begin{array}{c} |\cC_\aleph| \\ r \end{array} \right) r!$$ such functions.

Putting all of the above together we get
\begin{align*}
    P[\mathbf{Y} = (S_1,\dots,S_t)] \\
    = \left(\begin{array}{c} |\cC_\aleph| \\ n \end{array} \right)^{-1} \cdot \left(\begin{array}{c} |\mathcal{I}| \\ k \end{array} \right)^{-t} \cdot \left(\begin{array}{c} n \\ r \end{array} \right)^{-1}(r!)^{-1} \cdot \left(\begin{array}{c} |\cC_\aleph| \\ r \end{array} \right) r! \cdot \left( \begin{array}{c} |\cC_\aleph|-r \\ n-r \end{array} \right) \\
    = \left(\begin{array}{c} |\mathcal{I}| \\ k \end{array} \right)^{-t} = \left(\begin{array}{c} |\mathcal{C}_\aleph| \\ k \end{array} \right)^{-t}.
\end{align*}
This is the same as the probability of $P[\mathbf{X} = (S_1,\dots,S_t)]$, which finishes the proof. \qed
\end{proof}

%\fi

\bibliographystyle{plain}
\bibliography{ref}

\end{document}